\documentclass[journal]{IEEEtai}

\usepackage[colorlinks,urlcolor=blue,linkcolor=blue,citecolor=blue]{hyperref}

\usepackage{color,array}

\usepackage{graphicx}
\usepackage{amssymb}
\usepackage{amsmath}
\usepackage{amsthm}
\usepackage{algorithm}
\usepackage{algpseudocode}
\usepackage{booktabs} 
\usepackage{cite}
\usepackage{multirow}
\usepackage{tabularx}

\newtheorem{assumption}{Assumption}

\usepackage{booktabs,longtable,pdflscape,multirow}

\newtheorem{theorem}{Theorem}

\begin{document}

\title{\fontsize{22}{28}\selectfont Certified Uncertainty Propagation in One-Shot Federated Bayesian Models via Posterior Event Transport} 
 
\author{Mahyar Mohammadi, Mohammad Hossein Badiei, Abolfazl Yaghmaei, \and Hamed Kebriaei {$^\dagger$}, \IEEEmembership{Senior Member, IEEE}
\thanks{{$^\dagger$} Corresponding author: Hamed Kebriaei.}
\thanks{M. Mohammady, MH. Badiei, A. Yaghmaei and H. Kebriaei are with the School of Electrical and
Computer Engineering, College of Engineering, University of Tehran, Tehran, 1417614411, Iran, (email:
\href{mailto://mahyar.mohammady@ut.ac.ir}{mahyar.mohammady@ut.ac.ir}; \href{mailto://mh.badiei@ut.ac.ir}{mh.badiei@ut.ac.ir}; \href{mailto://yaghmaei@ut.ac.ir}{yaghmaei@ut.ac.ir};\href{mailto://kebriaei@ut.ac.ir} {kebriaei@ut.ac.ir}).
}
\thanks{H. Kebriaei is also with the School of Computer Science, Institute for Research in Fundamental Sciences (IPM), P.O. Box
19395-5746, Tehran, Iran. The work of Hamed Kebriaei was supported in part by the Institute for
Research in Fundamental Sciences (IPM) under Grant CS 1404-04-190.}}

\maketitle
\begin{abstract}
Probabilistic certification of Bayesian neural networks lower-bounds the posterior probability that a model satisfies a verifier-defined safety property. In one-shot federated Bayesian learning, however, the deployed model is obtained by aggregating parameters drawn from client-specific posterior distributions, and local certificates therefore do not directly provide safety guarantees for the aggregated model. This paper develops a deployment-consistent certification framework by propagating local posterior events through the deployment aggregation rule, with an exact geometric characterization for Federated Averaging(FedAvg). Each client constructs pairwise-disjoint hyper-rectangular regions in parameter space and computes their posterior probability masses. The server forms Cartesian products of these regions, maps them through the deployment rule, and retains the probability of a product event only when its complete aggregation image is verified to satisfy the prescribed safety property. Under independent client posteriors, the probability of each product event factorizes into the product of its local masses, and summing the probabilities of verified disjoint events yields a sound lower bound on the safety probability of the deployed model. For FedAvg with nonnegative aggregation coefficients, the image of a Cartesian product of axis-aligned hyper-rectangles is exactly represented by a weighted hyper-rectangle, introducing no additional set over-approximation during aggregation. We further distinguish the proposed transported-event certificate from direct certification under the posterior distributions induced by FedAvg and Product-of-Gaussians aggregation. We evaluate the framework on MNIST and Fashion-MNIST under label-Dirichlet heterogeneity as controlled one-shot federated benchmarks. The transported FedAvg certificate remains non-vacuous across all evaluated settings, ranging from \(22.51\%\) to \(46.89\%\), while direct global certificates under the evaluated global
posterior laws range from \(72.05\%\) to \(91.39\%\). The results further show that predictive accuracy and certifiable safety do not necessarily follow the same trend, and that different global posterior constructions can exhibit different certification behavior across architectures.
\end{abstract}

\begin{IEEEImpStatement}
This work advances trustworthy federated artificial intelligence by providing deployment-consistent safety guarantees for
one-shot Bayesian model aggregation. Existing local certificates do not directly characterize the safety of the final global model because aggregation changes the relevant probability distribution. The proposed framework closes this gap by transporting certified local posterior events through the actual deployment rule and assigning probability mass only to aggregation outcomes that are verified safe. The resulting guarantees are formulated for deployment aggregation rules and are exact at the set-propagation stage for FedAvg with nonnegative weights. By separating transported certification from direct certification under FedAvg and Product-of-Gaussians distributions, the framework also clarifies which safety probability each method certifies. The empirical analysis provides controlled insight into how heterogeneity, model architecture, client count, and verifier conservatism influence certifiable safety in one-shot federated settings.
\end{IEEEImpStatement}


\section{Introduction}

Federated learning (FL) enables multiple clients to train a shared model without centralizing their raw data, making it suitable for applications constrained by privacy, governance, and communication requirements. However, federated learning must operate under statistical heterogeneity, since clients often possess limited and non-identically distributed datasets. These challenges are particularly significant in one-shot FL, where each client communicates with the server only once and the deployed model is determined by a single aggregation step without further refinement 
\cite{mcmahan2017communication,kairouz2021advances,guha2019one}.

Bayesian neural networks (BNNs) provide a principled representation of epistemic uncertainty by maintaining posterior distributions over model parameters rather than single point estimates \cite{kendall2017uncertainties,blundell2015weight}. In federated settings, this allows each client to represent uncertainty induced by its local data. Nevertheless, uncertainty quantification alone does not provide a formal safety guarantee. For safety-critical applications, it is necessary to determine the probability that a model drawn from the distribution governing deployment satisfies a prescribed input–output property \cite{zhang2022personalized}.

These challenges are part of the broader effort toward trustworthy artificial
intelligence, where reliability, robustness, uncertainty, and certification
are considered essential properties of deployed AI systems \cite{rawal2021recent,pfau2025engineering}.
Probabilistic certification addresses this requirement by lower-bounding the posterior probability of models satisfying a verifier-defined safety specification. Existing methods for centralized BNNs construct verified regions in weight space and accumulate their posterior probability masses \cite{wicker2020probabilistic,wicker2024adversarial,batten2024tight}. These methods typically employ sound verification techniques, such as interval bound propagation or abstract interpretation, to certify bounded regions of inputs and model parameters \cite{gowal2018effectiveness,mirman2018differentiable}.  In the centralized setting, the relevant probability is defined under a single posterior over the deployed model.

The certification problem is fundamentally different in one-shot federated Bayesian learning. Each client learns a distinct local posterior, while the deployed model is produced by applying a server-side aggregation rule to jointly sampled client parameters. Consequently, the relevant deployment distribution depends on both the local posteriors and the aggregation map. Local certified probabilities cannot, in general, be directly summed or averaged to obtain a valid global guarantee. A local certificate describes an event in one client's parameter space, whereas deployment depends on a joint realization across all clients. Moreover, even when local parameter regions are individually safe, their aggregate may not remain safe because neural-network safety regions in weight space are generally nonconvex and need not be preserved by aggregation.

In this paper, we develop a deployment-consistent framework for probabilistic safety certification in one-shot federated Bayesian learning. Each client constructs pairwise-disjoint hyper-rectangular regions in parameter space and evaluates their posterior probability masses. The server forms joint events from combinations of these local regions, propagates them through the aggregation rule, and includes their probability mass only when the complete aggregation image is verified to satisfy the prescribed safety property. Under independent client posteriors, the probability of each joint event factorizes into the product of its local posterior masses. Summing the masses of verified disjoint events then yields a sound lower bound on the safety probability of the deployed model.
For FedAvg with nonnegative aggregation weights, the image of a Cartesian product of axis-aligned hyper-rectangles is represented exactly by the weighted combination of their endpoints. Therefore, the aggregation step introduces no additional set over-approximation, and the conservatism of the certificate arises only from finite posterior coverage and the underlying verification procedure. We also distinguish the proposed transported-event certificate from direct server-side certification. Direct certification under the FedAvg pushforward distribution targets the same deployment safety probability but may cover different regions of global parameter space. In contrast, Product-of-Gaussians aggregation defines a separate posterior fusion rule and is therefore treated as a distinct global-posterior baseline. 

The main contributions of our work are as follows:

\begin{itemize}
\item We formulate probabilistic safety for one-shot federated BNNs under the aggregation-induced distribution of the deployed model and clarify why local posterior certificates do not directly imply global safety.

\item We derive a sound local-to-global lower bound by propagating disjoint joint posterior events through the deployment aggregation rule and certifying their complete aggregated images.

\item We establish an exact hyper-rectangle-image characterization for FedAvg with nonnegative aggregation weights and evaluate transported and direct certificates on MNIST and Fashion-MNIST as controlled benchmarks for the communication-constrained one-shot setting.

\end{itemize}

\section{Related Work}
\label{sec:related_work}

Federated learning constructs a global model from locally trained client models without centralizing raw data, with FedAvg remaining one of the most widely used parameter-aggregation methods \cite{mcmahan2017communication,kairouz2021advances}. Bayesian extensions of federated learning replace or complement point estimates with posterior distributions to represent uncertainty caused by limited and heterogeneous client data. Recent studies have explored Bayesian formulations of federated learning to
improve uncertainty representation and robustness under heterogeneous client
data distributions \cite{yu2025pfedbl,li2025federated}. Bayesian nonparametric federated learning addresses permutation inconsistencies among independently trained neural networks by matching and merging local network components \cite{yurochkin2019bayesian}, while posterior averaging interprets federated aggregation as approximate Bayesian inference \cite{al2020federated}. FedBE instead fits a server-side distribution over client models and uses samples from this distribution for Bayesian model ensembling \cite{chen2020fedbe}. More recent approaches combine predictive distributions, layer-wise posterior approximations, or local Laplace approximations to improve predictive accuracy and uncertainty calibration under statistical heterogeneity \cite{hasan2024calibrated,liu2024fedlpa}. Although these methods provide posterior-aware global predictors, they do not certify the probability that the random model produced by the deployment rule satisfies a verifier-defined safety specification.

This limitation is particularly relevant in one-shot federated learning, where each client communicates with the server only once and the aggregation rule directly determines the deployed model \cite{guha2019one}. Existing one-shot methods primarily use parameter aggregation, ensembling, posterior approximation, or knowledge distillation to construct an accurate global predictor from independently trained client models. In Bayesian one-shot learning, heterogeneous local datasets induce distinct posterior distributions that must be combined without subsequent communication or correction. Prior work therefore focuses mainly on predictive performance, calibration, and uncertainty representation. In contrast, our objective is to characterize the probability that the model induced by a general one-shot aggregation map satisfies a prescribed safety property.

Probabilistic verification of Bayesian neural networks addresses safety by measuring the posterior probability assigned to parameter configurations satisfying a specified input--output property. Wicker et al. derive sound lower bounds by constructing verified regions in weight space and integrating their posterior probability masses \cite{wicker2020probabilistic}. Related studies develop statistical robustness guarantees \cite{cardelli2019statistical}, unified lower and upper probabilistic certificates \cite{wicker2024adversarial}, tighter region-based bounds \cite{batten2024tight}, and dynamic-programming methods for Bayesian neural-network robustness \cite{adams2023bnn}. These approaches commonly rely on sound verification techniques, including interval bound propagation and abstract interpretation, to certify bounded sets of inputs and model parameters \cite{gowal2018effectiveness,mirman2018differentiable}. However, they generally assume a single posterior distribution over the deployed model and do not directly address deployment mechanisms that aggregate random parameters drawn from several client-specific posteriors.

Our work focuses on certifying the global model obtained from local Bayesian posteriors, rather than certifying each local posterior separately. For FedAvg, we consider both transported and direct certification under the induced global distribution. Product-of-Gaussians is included as a separate fusion baseline, allowing us to distinguish between different notions of global Bayesian safety.

\section{Methodology}
\label{sec:methodology}

Consider a one-shot federated learning system comprising $n$ clients. Client $i \in \{1, \dots, n\}$ holds a private dataset $\mathcal{D}_i$ and independently trains a local Bayesian neural network. Following the local training phase, the posterior representation is transmitted from each client to a central server, limited to a single occurrence. The server then constructs the deployed global model through a single aggregation step. As an initial step, the derivation is presented for a single safety property. The same procedure is independently applied to every evaluated property, and the resulting certified probabilities are subsequently averaged.

\subsection{Safety Certification Objective}
\label{subsec:safety_objective}

For the classification-robustness property considered in the experiments, let $x \in [0, 1]^p$ be a reference input of dimension $p$ with class label $c$, and define the admissible perturbation region as
$
\mathcal{X}
:=
\left\{
x'\in[0,1]^p:
\left\lVert x'-x\right\rVert_{\infty}
\leq
\epsilon
\right\},
\label{eq:input_region}
$
    where ($\epsilon \ge 0$) denotes the input perturbation radius. A parameter vector ($\theta \in \mathbb{R}^{n_w}$), where ($n_w$) denotes the total number of trainable parameters in the neural network, is considered safe if the following condition holds.(For compactness, we denote the vector of all ($K$) class-specific outputs by ($f_\theta(x') := [f_{\theta,1}(x'),\ldots,f_{\theta,K}(x')]^{\mathsf T}\in\mathbb{R}^{K})$)
\begin{equation}
C_{\mathrm{s}}f_{\theta}(x')
+
d_{\mathrm{s}}
\succeq
0,
\qquad
\forall x'\in\mathcal{X},
\label{eq:safety_property}
\end{equation}
where $C_s \in \mathbb{R}^{n_s \times K}$ and $d_s \in \mathbb{R}^{n_s}$ define a set of  $n_s$ linear inequality constraints on the $K-dimensional$ classification logits, with the inequality understood in the component-wise sense. Accordingly, the safe parameter set $\mathcal{S}$ consists of all network
parameters whose outputs satisfy the above constraints for every input in
$\mathcal{X}$.
This constraint is equivalent to requiring that the logit of the true class $c$, denoted $f_{\theta,c}(x')$, exceeds the logit of any competing class $r$, denoted $f_{\theta,r}(x')$, by at least a margin $\delta \ge 0$:
\begin{equation}
f_{\theta,c}(x')
-
f_{\theta,r}(x')
\geq
\delta,
\qquad
\forall x'\in\mathcal{X},
\quad
\forall r\neq c.
\label{eq:equivalent_margin_constraint}
\end{equation}
Let $\Theta_i$ denote a random weight vector drawn from the posterior of client $i$. The server applies an aggregation rule
\begin{equation}
\Theta_g
=
\mathcal{A}
(\Theta_1,\ldots,\Theta_n),
\label{eq:general_aggregation}
\end{equation}
where $\Theta_g$ represents the deployed random global model. The global safety probability of interest is therefore
\begin{equation}
P_{\mathrm{safe}}^{g}
:=
\Pr
\left(
\Theta_g\in\mathcal{S}
\right).
\label{eq:global_safety_goal}
\end{equation}
The objective is to construct a computable certificate $L_{safe}^g$ such that
\begin{equation}
L_{\mathrm{safe}}^{g}
\leq
P_{\mathrm{safe}}^{g}.
\label{eq:lower_bound_objective}
\end{equation}

\subsection{Local Bayesian Posterior}
\label{subsec:local_posterior}

Each client $i$ minimizes a task-specific loss function $l(\theta_i; \mathcal{D}_i)$ to obtain a local maximum a posteriori (MAP) estimate, denoted by $\mu_i$. Parameter uncertainty is represented by a Gaussian posterior:
\begin{equation}
q_i(\theta_i)
=
p(\theta_i\mid\mathcal{D}_i)
\approx
\mathcal{N}
\left(
\mu_i,\Sigma_i
\right).
\label{eq:local_gaussian_posterior}
\end{equation}
The Gaussian approximation may be obtained using a Laplace approximation,
variational inference, or another posterior approximation method. The
certification principle developed below is not restricted to a particular
Bayesian training procedure.In the implementation, a diagonal covariance is used:
\begin{equation}
\Sigma_i
=
\operatorname{diag}
\left(
\sigma_{i,1}^{2},
\ldots,
\sigma_{i,n_w}^{2}
\right).
\label{eq:diagonal_covariance}
\end{equation}
This assumption enables the probability mass of an axis-aligned
hyper-rectangle to be evaluated in closed form.
\begin{assumption}[\textbf{Independent Deployment Draws}]
\label{ass:independent_posteriors}
At deployment, each client’s local parameters are drawn independently, where $\Theta_i\sim q_i$ and 
$q(\theta_1,\dots,\theta_n)=\prod_{i=1}^n q_i(\theta_i)$.
\end{assumption}

Assumption~\ref{ass:independent_posteriors} allows the probability of any joint
event across clients to be factorized as the product of the corresponding
local posterior probabilities. When the joint probabilities are directly
available, the set‑based safety argument remains valid without this
independence assumption.

\subsection{Local Hyper-rectangle Construction}
\label{subsec:local_cells}

Each client constructs a finite collection of candidate regions in its local
weight space. For a posterior sample \(w_{i,k}^{*}\sim\mathcal{N}(\mu_i,\Sigma_i)\),
client $i$ forms the hyper-rectangle:
\begin{equation}
H_{i,k}
=
\prod_{d=1}^{n_w}
\left[
w_{i,k,d}^{*}
-
\gamma\sigma_{i,d},
;
w_{i,k,d}^{*}
+
\gamma\sigma_{i,d}
\right],
\label{eq:local_hyperrectangle}
\end{equation}
where $\gamma > 0$ controls the expansion factor defining the size of the candidate region. Sampling the center from the posterior directs the search toward regions of relatively high probability mass. Scaling the interval widths by the local posterior standard deviations also accounts for differences in uncertainty across model parameters.

Each client retains a pairwise-disjoint collection
\(\mathcal{C}_i=\{H_{i,1},\ldots,H_{i,k}\}\).
Pairwise disjointness is imposed because posterior masses of overlapping
regions cannot be added directly without correcting for repeated mass.
Rejecting an overlapping candidate may reduce certified coverage but does
not affect the validity of the final lower bound.

For a diagonal Gaussian posterior and a hyper-rectangle denoted by its lower and upper dimensional bounds $l_{i,k,d}$ and $u_{i,k,d}$:
\begin{equation}
H_{i,k}
=
\prod_{d=1}^{n_w}
[\ell_{i,k,d},u_{i,k,d}],
\end{equation}
its local posterior mass is
\begin{equation}
\begin{aligned}
p_{i,k}
:= q_i(H_{i,k}),
\end{aligned}
\label{eq:local_cell_probability}
\end{equation}
Equivalently, the mass may be written using the error function (erf) as:
\begin{equation}
p_{i,k}
= \prod_{d=1}^{n_w} \frac{1}{2}
\biggl[
\operatorname{erf}\!\left( \frac{u_{i,k,d} - \mu_{i,d}}{\sqrt{2}\,\sigma_{i,d}} \right)
-
\operatorname{erf}\!\left( \frac{\ell_{i,k,d} - \mu_{i,d}}{\sqrt{2}\,\sigma_{i,d}} \right)
\biggr].
\label{eq:local_cell_probability_erf}
\end{equation}

\subsection{Safety Verification via IBP}
\label{subsec:ibp_verification}

Interval bound propagation (IBP) is used to verify bounded input and weight regions. For scalar weights $w \in [w^L, w^U]$ and inputs $x \in [x^L, x^U]$, the output interval $[t_{min}, t_{max}]$ of their product is given by:
\begin{align}
t_{\min} &= \min\{ w^{L}x^{L},\; w^{L}x^{U},\; w^{U}x^{L},\; w^{U}x^{U} \}, \label{eq:ibp_min} \\
t_{\max} &= \max\{ w^{L}x^{L},\; w^{L}x^{U},\; w^{U}x^{L},\; w^{U}x^{U} \}. \label{eq:ibp_max}
\end{align}
The interval bounds are propagated recursively through all network layers,
yielding certified lower and upper bounds on each output logit. For a
candidate parameter hyper-rectangle \(H\subseteq\mathbb{R}^{n_w}\) and an
admissible input set \(\mathcal{X}\), let $\underline{z}_{r}(H,\mathcal{X})$ and $\overline{z}_{r}(H,\mathcal{X})$ denote the IBP lower and upper bounds, respectively, for the output logit
associated with class \(r\). These bounds satisfy
\begin{equation}
    \underline{z}_{r}(H,\mathcal{X})
    \leq
    f_{\theta,r}(x')
    \leq
    \overline{z}_{r}(H,\mathcal{X}),
    \quad
    \forall \theta\in H,
    \forall x'\in\mathcal{X}.
    \label{eq:ibp_output_enclosure}
\end{equation}
Let \(c\) denote the reference class. The parameter region \(H\) is certified
safe whenever the smallest possible logit of class \(c\) remains at least
\(\delta\) larger than the largest possible logit of every competing class:
\begin{equation}
    \underline{z}_{c}(H,\mathcal{X})
    -
    \overline{z}_{r}(H,\mathcal{X})
    \geq
    \delta,
    \qquad
    \forall r\in\{1,\ldots,K\}\setminus\{c\}.
    \label{eq:ibp_verification_condition}
\end{equation}
Therefore, every network parameterized by a weight vector in \(H\) assigns
class \(c\) a logit margin of at least \(\delta\) over all competing classes
for every admissible input in \(\mathcal{X}\).

Local IBP verification may be used to screen candidate cells before they are
communicated. However, local safety alone does not imply global safety after
aggregation. Therefore, each combined global region must be verified again
after the aggregation step.

\begin{assumption}[\textbf{Correctness of the Verifier}]
\label{ass:verifier_soundness}
For any parameter region $H$, if the verifier declares $H$ to be safe (denoted by $\text{SAFE}(H) = 1$), then every parameter vector in $H$ satisfies the safety requirement:
\begin{equation}
    \operatorname{SAFE}(H)=1
    \quad\Longrightarrow\quad
    H\subseteq\mathcal{S},
    \label{eq:sound_verifier}
\end{equation}
where \(\mathcal{S}\) denotes the true safe parameter set.
\end{assumption}

The verifier may reject some regions that are actually safe because its
bounds are conservative. This may reduce the certified probability and make
the resulting lower bound less tight. However, every region accepted by the
verifier is guaranteed to be safe.

\subsection{One-Shot FedAvg Deployment}
\label{subsec:fedavg_aggregation}

We instantiate the general deployment map using FedAvg:
\begin{equation}
\Theta_g
=
\sum_{i=1}^{n}
\alpha_i\Theta_i,
\qquad
\alpha_i\geq0,
\qquad
\sum_{i=1}^{n}\alpha_i=1.
\label{eq:fedavg_aggregation}
\end{equation}
The coefficients $\alpha_i$ may be uniform or proportional to the local dataset sizes. Consider one selected cell from every client, indexed by the tuple
\(k=(k_1,\ldots,k_n)\).
The corresponding joint event is
\(E_k=\{\Theta_1\in H_{1,k_1},\ldots,\Theta_n\in H_{n,k_n}\}\).
Under Assumption~\ref{ass:independent_posteriors}, its probability is:
\begin{equation}
\Pr(E_{k})
=
\prod_{i=1}^{n}
p_{i,k_i}.
\label{eq:joint_event_probability}
\end{equation}
Because the FedAvg weights are nonnegative, the complete image of the joint
hyper-rectangle is exactly:
\begin{equation}
G_{k}
=
\prod_{d=1}^{n_w}
\left[
\sum_{i=1}^{n}
\alpha_i\ell_{i,k_i,d},
\sum_{i=1}^{n}
\alpha_i u_{i,k_i,d}
\right].
\label{eq:global_aggregated_cell}
\end{equation}
Equation~\eqref{eq:global_aggregated_cell} is exact rather than an outer
approximation. Thus, the FedAvg operation introduces no additional interval
relaxation when the local regions are axis-aligned hyper-rectangles and the
aggregation weights are nonnegative.
A tuple is accepted only when its complete aggregated image satisfies
$
\operatorname{SAFE}
(G_{k})
=
1.
$
Consider the following:
\begin{equation}
\mathcal{K}_{\mathrm{safe}}
=
\left\{
k:
\operatorname{SAFE}
(G_{k})=1
\right\}
\label{eq:safe_tuple_set}
\end{equation}
denote the set of accepted tuples.

\subsection{Deployment-Consistent Safety Certificate}
\label{subsec:global_certificate}

The global safe set is the verifier-defined set $\mathcal{S}$. It is not defined as the union of the local certified sets; instead, local regions are combined through the deployment rule, and their aggregated images are verified directly. 

\begin{theorem}[\textbf{FedAvg Safety Lower Bound}]
\label{thm:transported_certificate}
Suppose that the client posterior samples are independent, the retained local
weight regions are pairwise disjoint, and the verifier satisfies
Assumption~\ref{ass:verifier_soundness}. Then,
\begin{equation}
    L_{\mathrm{safe}}^{\mathrm{FA}}
    :=
    \sum_{\boldsymbol{k}\in\mathcal{K}_{\mathrm{safe}}}
    \prod_{i=1}^{n} p_{i,k_i}
    \label{eq:transported_certificate}
\end{equation}
is a valid lower bound on the probability that the model deployed by FedAvg is
safe. In particular,
\begin{equation}
    L_{\mathrm{safe}}^{\mathrm{FA}}
    \leq
    \Pr\!\left(
        \sum_{i=1}^{n}\alpha_i\Theta_i
        \in\mathcal{S}
    \right)
    =
    P_{\mathrm{safe}}^{g}
    \label{eq:main_certificate}
\end{equation}
where \(\mathcal{K}_{\mathrm{safe}}\) denotes the set of client-region
combinations whose complete FedAvg images are certified safe, and
\(p_{i,k_i}\) is the posterior probability of the selected region at client
\(i\).
\end{theorem}

\begin{proof}
For each tuple
\(k=(k_1,\ldots,k_n)\in\mathcal{K}_{\mathrm{safe}}\),
consider the joint event
\begin{equation}
    E_{k}
    :=
    \left\{
        \Theta_i\in H_{i,k_i},
        i=1,\ldots,n
    \right\},
    \label{eq:joint_safe_event}
\end{equation}
where \(H_{i,k_i}\) is the selected parameter region of client \(i\).
By the definition of \(\mathcal{K}_{\mathrm{safe}}\), the complete FedAvg
image of the regions associated with \(k\) has been verified to
lie inside the safe set \(\mathcal{S}\). Hence, whenever
\(E_{k}\) occurs, the aggregated parameter vector satisfies:
\begin{equation}
    \sum_{i=1}^{n}\alpha_i\Theta_i
    \in
    \mathcal{S}.
\end{equation}
Therefore,
\begin{equation}
    E_{k}
    \subseteq
    \left\{
        \sum_{i=1}^{n}\alpha_i\Theta_i
        \in\mathcal{S}
    \right\},
    \qquad
    \forall
    k\in\mathcal{K}_{\mathrm{safe}}.
    \label{eq:safe_event_inclusion}
\end{equation}
Taking the union over all verified tuples gives
\begin{equation}
    \bigcup_{k\in\mathcal{K}_{\mathrm{safe}}}
    E_{k}
    \subseteq
    \left\{
        \sum_{i=1}^{n}\alpha_i\Theta_i
        \in\mathcal{S}
    \right\}.
    \label{eq:verified_union_inclusion}
\end{equation}
Consequently,
\begin{equation}
    \Pr\left(
        \bigcup_{k\in\mathcal{K}_{\mathrm{safe}}}
        E_{k}
    \right)
    \leq
    \Pr\left(
        \sum_{i=1}^{n}\alpha_i\Theta_i
        \in\mathcal{S}
    \right).
    \label{eq:probability_inclusion}
\end{equation}
Because the retained parameter regions of each client are pairwise disjoint,
two distinct tuples correspond to disjoint joint events. Thus,
\begin{equation}
    \Pr\left(
        \bigcup_{k\in\mathcal{K}_{\mathrm{safe}}}
        E_{k}
    \right)
    =
    \sum_{k\in\mathcal{K}_{\mathrm{safe}}}
    \Pr(E_{k}).
    \label{eq:disjoint_event_sum}
\end{equation}
Moreover, under the independence of the client posterior draws,
the probability of each joint event factorizes as
\begin{align}
    \Pr(E_{k})
    &=
    \Pr\left(
        \Theta_1\in H_{1,k_1},
        \ldots,
        \Theta_n\in H_{n,k_n}
    \right) \nonumber\\
    &=
    \prod_{i=1}^{n}
    \Pr\left(
        \Theta_i\in H_{i,k_i}
    \right) \nonumber\\
    &=
    \prod_{i=1}^{n}p_{i,k_i}.
    \label{eq:joint_event_factorization}
\end{align}

Substituting~\eqref{eq:joint_event_factorization} into
\eqref{eq:disjoint_event_sum} yields
\begin{equation}
    \Pr\left(
        \bigcup_{k\in\mathcal{K}_{\mathrm{safe}}}
        E_{k}
    \right)
    =
    \sum_{k\in\mathcal{K}_{\mathrm{safe}}}
    \prod_{i=1}^{n}p_{i,k_i}
    =
    L_{\mathrm{safe}}^{\mathrm{FA}}.
    \label{eq:certificate_as_probability}
\end{equation}

Combining~\eqref{eq:probability_inclusion} and
\eqref{eq:certificate_as_probability}, we obtain
\begin{equation}
    L_{\mathrm{safe}}^{\mathrm{FA}}
    \leq
    \Pr\left(
        \sum_{i=1}^{n}\alpha_i\Theta_i
        \in\mathcal{S}
    \right)
    =
    P_{\mathrm{safe}}^{g}.
\end{equation}
Hence, \(L_{\mathrm{safe}}^{\mathrm{FA}}\) is a valid lower bound on the
deployment safety probability.
\end{proof} 

Algorithm \ref{alg:certification} summarizes the complete operational procedure for executing this framework. 

\textbf{Algorithm Explanation:} The algorithmic procedure is divided into a parallelizable client-side discretization phase and a server-side aggregation phase. During the local phase (Lines 1-12), each client $i$ iterates up to a sampling budget $N$, drawing a stochastic center point $w_{i,k}^*$ directly from its posterior distribution $q_i$ (Lines 3-4). A hyperrectangle $H_{i,k}$ is then constructed by scaling the interval widths around the center by the local posterior standard deviations using the expansion factor $\gamma$ (Line 5). To ensure probabilities can be summed linearly without inclusion-exclusion corrections, the algorithm checks if the proposed region intersects with any previously stored sets in $\mathcal{C}_i$ (Line 7). If it is completely disjoint, the exact probability mass $p_{i,k}$ is evaluated and the region is stored (Lines 8-9). In the server phase (Lines 13-22), the server constructs a set $\mathcal{K}$ of combination tuples representing the Cartesian product of local regions provided by the clients (Line 13). For every unique combination tuple $k$, the server constructs the complete aggregate image $G_k$ using the FedAvg map (Lines 15-16). The verifier checks if the entire aggregated geometric region is safe (Line 17). If accepted, the tuple is added to the safe tracking set $\mathcal{K}_{safe}$, and its joint probability which has been calculated as the independent product of local masses, is accumulated into the total safety bound $L_{safe}^{FA}$ (Lines 18-19).
	
\begin{algorithm}[t]
		\caption{Deployment-Consistent FedAvg Safety Certificate}
		\label{alg:certification}
		\begin{algorithmic}[1]
			\Require Local posteriors $\{q_i\}_{i=1}^n$, FedAvg coefficients $\{\alpha_i\}_{i=1}^n$, input region $\mathcal{X}$, expansion factor $\gamma$, local sampling budget $N$
			\Ensure Certified lower bound $L_{safe}^{FA}$
			\For{$i=1 \dots n$}
			\State $\mathcal{C}_i \leftarrow \emptyset$
			\For{$k=1 \dots N$}
			\State Sample $w_{i,k}^* \sim q_i$
			\State Construct $H_{i,k}$ using Eq. \eqref{eq:local_hyperrectangle}
			\State Optionally screen $H_{i,k}$ using local IBP
			\If{$H_{i,k}$ is retained and disjoint from all regions in $\mathcal{C}_i$}
			\State $p_{i,k} \leftarrow q_i(H_{i,k})$
			\State $\mathcal{C}_i \leftarrow \mathcal{C}_i \cup \{(H_{i,k}, p_{i,k})\}$
			\EndIf
			\EndFor
			\EndFor
			\State Construct a set $\mathcal{K}$ of at most $B$ unique tuples; enumerate all tuples when feasible, otherwise prioritize them by $\prod_{i=1}^n p_{i,k_i}$
			\State $\mathcal{K}_{safe} \leftarrow \emptyset$, $L_{safe}^{FA} \leftarrow 0$
			\For{\textbf{each} $k = (k_1, \dots, k_n) \in \mathcal{K}$}
			\State Construct $G_k$ using Eq. \eqref{eq:global_aggregated_cell}
			\If{$\text{SAFE}(G_k) == 1$}
			\State $\mathcal{K}_{safe} \leftarrow \mathcal{K}_{safe} \cup \{k\}$
			\State $L_{safe}^{FA} \leftarrow L_{safe}^{FA} + \prod_{i=1}^n p_{i,k_i}$
			\EndIf
			\EndFor
        
        \Return $L_{safe}^{FA}$
		\end{algorithmic}
	\end{algorithm}

\subsection{Direct Certification Under Global Posterior Laws}
\label{subsec:direct_global_certification}

In addition to the transported certificate, we consider direct certification
under a global distribution defined in the deployed parameter space. This
approach constructs and verifies weight regions directly under the selected
global posterior, rather than transporting combinations of local regions
through the aggregation map.
Assume that the client parameters are independently distributed as
\begin{equation}
    \Theta_i
    \sim
    \mathcal{N}(\mu_i,\Sigma_i),
    \qquad
    i=1,\ldots,n.
    \label{eq:local_gaussian_assumption}
\end{equation}
For FedAvg,
\begin{equation}
    \Theta_g^{\mathrm{FA}}
    =
    \sum_{i=1}^{n}\alpha_i\Theta_i,
    \qquad
    \sum_{i=1}^{n}\alpha_i=1,
    \label{eq:fedavg_random_model}
\end{equation}
the induced global distribution is Gaussian:
\begin{equation}
    q_g^{\mathrm{FA}}
    =
    \mathcal{N}
    \left(
        \mu_g^{\mathrm{FA}},
        \Sigma_g^{\mathrm{FA}}
    \right),
    \label{eq:fedavg_global_posterior}
\end{equation}
with
\begin{equation}
		\mu_g^{FA} = \sum_{i=1}^n \alpha_i \mu_i, \quad \Sigma_g^{FA} = \sum_{i=1}^n \alpha_i^2 \Sigma_i. \label{eq:fedavg_params}
\end{equation}

Let
\(\{C_{j,t}^{\mathrm{FA}}\}_{t=1}^{T_j}\) be a collection of pairwise-disjoint
global hyper-rectangles constructed under \(q_g^{\mathrm{FA}}\) for safety
property \(j\). Each region is retained only if it is verified to satisfy the
prescribed safety specification over its complete extent. The corresponding
direct FedAvg certificate is
\begin{equation}
    L_{\mathrm{safe},j}^{\mathrm{dir,FA}}
    :=
    \sum_{t=1}^{T_j}
    q_g^{\mathrm{FA}}
    \left(
        C_{j,t}^{\mathrm{FA}}
    \right)
    \leq
    q_g^{\mathrm{FA}}(\mathcal{S}_j)
    =
    P_{\mathrm{safe},j}^{g}.
    \label{eq:direct_fedavg_certificate}
\end{equation}
The transported and direct FedAvg certificates therefore lower-bound the same
deployment-level safety probability. They differ, however, in the regions
used to construct the bound. The transported certificate begins with events
under the local client posteriors and verifies their images after aggregation,
whereas the direct certificate searches for safe regions directly under the
FedAvg-induced global distribution. Since these procedures generally cover
different subsets of the relevant probability space, neither certificate is
guaranteed to dominate the other.

We also consider Product-of-Gaussians (PoG) fusion as a separate global
posterior baseline:
\begin{equation}
    q_g^{\mathrm{PoG}}(\theta)
    \propto
    \prod_{i=1}^{n}q_i(\theta).
    \label{eq:pog_posterior}
\end{equation}
For Gaussian local factors, the resulting distribution is
\begin{equation}
    q_g^{\mathrm{PoG}}
    =
    \mathcal{N}
    \left(
        \mu_g^{\mathrm{PoG}},
        \Sigma_g^{\mathrm{PoG}}
    \right),
    \label{eq:pog_global_posterior}
\end{equation}
where
\begin{align}
    \Sigma_g^{\mathrm{PoG}}
    &=
    \left(
        \sum_{i=1}^{n}\Sigma_i^{-1}
    \right)^{-1},
    \label{eq:pog_covariance}\\
    \mu_g^{\mathrm{PoG}}
    &=
    \Sigma_g^{\mathrm{PoG}}
    \sum_{i=1}^{n}
    \Sigma_i^{-1}\mu_i.
    \label{eq:pog_mean}
\end{align}

Let
\(\{C_{j,t}^{\mathrm{PoG}}\}_{t=1}^{T_j^{\mathrm{PoG}}}\) denote
pairwise-disjoint regions constructed under \(q_g^{\mathrm{PoG}}\) and
verified safe for property \(j\). The direct PoG certificate is
\begin{equation}
    L_{\mathrm{safe},j}^{\mathrm{dir,PoG}}
    :=
    \sum_{t=1}^{T_j^{\mathrm{PoG}}}
    q_g^{\mathrm{PoG}}
    \left(
        C_{j,t}^{\mathrm{PoG}}
    \right)
    \leq
    q_g^{\mathrm{PoG}}(\mathcal{S}_j).
    \label{eq:direct_pog_certificate}
\end{equation}

The PoG certificate characterizes safety under the PoG fusion distribution.
It should not be interpreted as a certificate for FedAvg deployment, since
FedAvg and PoG define different random global models and, consequently,
different safety probabilities.

\section{Experiments}
\label{sec:experiments}

We conduct a controlled evaluation of three quantities that have distinct
probabilistic meanings: direct certification under a selected global posterior,
the transported FedAvg certificate derived from local posterior events, and
Monte Carlo acceptance of the input verifier. The experiments examine whether
the transported certificate is non-vacuous, how certification varies with
architecture and data partitioning, and whether FedAvg and
Product-of-Gaussians (PoG) fusion exhibit consistent differences.

\subsection{Experimental setup}
\label{subsec:experimental_setup}

We evaluate MNIST \cite{lecun1998gradient} and Fashion-MNIST
\cite{xiao2017fashion} using random subsets of 12,000 training examples and
2,000 test examples. The training subset is partitioned among
\(n\in\{2,3,5\}\) clients using label-Dirichlet sampling
\cite{hsu2019measuring} with
\[
    \alpha\in\{0.5,0.6,0.7,0.9,1,2,5,7,10\}.
\]
Smaller values of \(\alpha\) produce more label-skewed client datasets. All
clients share the same network parameterization and initialization. We consider
fully connected Bayesian neural networks with hidden architectures
\(1\times64\), \(1\times128\), and \(2\times64\), using ReLU activations.

Each client trains a diagonal mean-field posterior for five epochs using Adam
\cite{kingma2014adam}, a learning rate of \(10^{-3}\), batch size 128, a
standard-normal prior, and a Bayes-by-Backprop objective with KL coefficient
\(10^{-4}\) \cite{blundell2015weight}. The post-training diagonal standard deviation is
fixed to \(10^{-5}\). The experiment is therefore designed as a controlled study of aggregation geometry and certificate construction rather than a study of posterior calibration. The selected posterior scale enables a consistent analysis of deployment-consistent certification behavior across federated configurations. Implementation parameters are detailed in Table \ref{tab:experimental_protocol}.

\begin{table}[h]
\centering
\caption{Implementation parameters used in the reported experiments.}
\label{tab:experimental_protocol}
\scriptsize
\setlength{\tabcolsep}{3.2pt}
\renewcommand{\arraystretch}{1.08}
\begin{tabular}{lr@{\qquad}lr}
\toprule
Training subset & 12,000
& Test subset & 2,000\\
Local epochs & 5
& Batch size & 128\\
Learning rate & \(10^{-3}\)
& KL coefficient & \(10^{-4}\)\\
Posterior scale & \(10^{-5}\)
& Properties/config. & 50\\
Input radius & \(10^{-3}\)
& Logit margin & 0\\
Centers/\(\gamma\) & 200
& \(\gamma\) values & 3--7\\
Direct MC/property & 300
& Aggregate MC/property & 3,000\\
Cells/client/property & 8
& Tuple budget/property & 20,000\\
\bottomrule
\end{tabular}
\end{table}

We compare two global laws. Equal-weight FedAvg uses
\(\alpha_i=1/n\) and the corresponding Gaussian pushforward. PoG uses the
uncorrected product of the local Gaussian approximations and is treated as a
precision-fusion baseline. Direct certificates are constructed independently
under both global laws. The transported local-event certificate is reported
only for FedAvg, because the implemented interval transport uses the FedAvg
parameter map.
For each configuration, 50 correctly classified test inputs define pointwise
robustness properties with input radius \(\epsilon=10^{-3}\) and logit margin
\(\delta=0\). Candidate weight boxes are generated from 200 posterior samples
for each
\[
    \gamma\in\{3,4,5,6,7\}.
\]
At most eight pairwise-disjoint cells are retained per client and property.
The FedAvg transport procedure evaluates at most 20,000 unique cell tuples per
property. Direct MC--IBP uses 300 global posterior samples; the separate
aggregated-sampling diagnostic uses 3,000 samples.

Accuracy is evaluated using the posterior-mean network. For a selected global
law, the direct certificate is the posterior mass of the disjoint global boxes
verified by IBP. The quantity stored by the implementation as
\texttt{empirical\_safety} is reported here as MC--IBP: it is the fraction of
sampled weight vectors whose complete input region is accepted by IBP. It is
not an exact safety probability and is not a formal upper bound.
All property-level quantities are first macro-averaged within each
configuration. Tables report the mean and sample standard deviation across the
available client-count and Dirichlet configurations. These deviations describe
variation across federation settings; they are not confidence intervals over
independent training repetitions.

\subsection{Simulation results and discussion}
\label{subsec:results}

Table~\ref{tab:heterogeneity_grouped} reports predictive accuracy, local certification,
direct global certification, MC--IBP acceptance, and runtime for PoG and
equal-weight FedAvg. The transported certificate is reported only for FedAvg,
because the implemented local-event transport follows the FedAvg parameter
map. All probability values are macro-averaged over the selected robustness
properties and are reported as percentages. The standard deviations summarize
variation across client-count and Dirichlet-partition configurations rather
than repeated training seeds.

\begin{table*}[t]
\centering
\caption{PoG and FedAvg results under different heterogeneity regimes. 
Non-IID and IID groups are formed according to the Dirichlet concentration parameter, and entries report mean $\pm$ standard deviation across the corresponding configurations.}
\label{tab:heterogeneity_grouped}
\scriptsize
\setlength{\tabcolsep}{1.55pt}
\renewcommand{\arraystretch}{0.96}
\resizebox{\textwidth}{!}{%
\begin{tabular}{llclrrrrr@{\hspace{4pt}}rrrrrr}
\toprule
& & & & \multicolumn{5}{c}{\textbf{PoG}} & \multicolumn{6}{c}{\textbf{FedAvg}} \\
\cmidrule(lr){5-9}\cmidrule(lr){10-15}
\textbf{Arch.} & \textbf{Dataset} & \textbf{Clients} & \textbf{Heterogeneity} & \textbf{Acc.(\%)} & $\boldsymbol{L^{\mathrm{loc}}(\%)}$ & $\boldsymbol{L^{\mathrm{dir}}(\%)}$ & \textbf{MC--IBP(\%)} & \textbf{Time(s)} & \textbf{Acc.(\%)} & $\boldsymbol{L^{\mathrm{loc}}(\%)}$ & $\boldsymbol{L^{\mathrm{tr}}(\%)}$ & $\boldsymbol{L^{\mathrm{dir}}(\%)}$ & \textbf{MC--IBP(\%)} & \textbf{Time(s)} \\
\midrule
\multirow{12}{*}{\(1\times64\)} & \multirow{6}{*}{MNIST} & \multirow{2}{*}{2} & Non-IID & \(70.11\pm5.68\) & \(74.04\pm5.01\) & \(92.54\pm2.77\) & \(96.67\pm2.89\) & \(333.46\pm14.09\) & \(70.11\pm5.68\) & \(74.04\pm5.01\) & \(52.85\pm11.63\) & \(92.54\pm2.77\) & \(96.67\pm2.89\) & \(330.97\pm11.57\) \\
 &  &  & IID & \(86.96\pm0.34\) & \(94.21\pm2.41\) & \(92.56\pm2.74\) & \(96.67\pm2.89\) & \(362.41\pm6.08\) & \(86.96\pm0.34\) & \(94.21\pm2.41\) & \(88.67\pm4.67\) & \(92.56\pm2.74\) & \(96.67\pm2.89\) & \(361.75\pm5.46\) \\
\addlinespace[0.5pt]
 &  & \multirow{2}{*}{3} & Non-IID & \(66.31\pm7.29\) & \(61.64\pm1.92\) & \(95.73\pm0.00\) & \(100.00\pm0.00\) & \(411.81\pm4.66\) & \(66.31\pm7.29\) & \(61.64\pm1.92\) & \(14.95\pm14.36\) & \(95.73\pm0.00\) & \(100.00\pm0.00\) & \(411.93\pm3.88\) \\
 &  &  & IID & \(83.93\pm0.61\) & \(90.60\pm3.73\) & \(94.11\pm2.81\) & \(100.00\pm0.00\) & \(482.04\pm9.55\) & \(83.93\pm0.61\) & \(90.60\pm3.73\) & \(76.25\pm7.91\) & \(94.11\pm2.81\) & \(100.00\pm0.00\) & \(478.63\pm9.05\) \\
\addlinespace[0.5pt]
 &  & \multirow{2}{*}{5} & Non-IID & \(59.51\pm10.87\) & \(58.81\pm10.18\) & \(89.32\pm7.40\) & \(96.67\pm5.77\) & \(596.03\pm49.09\) & \(59.51\pm10.87\) & \(58.81\pm10.18\) & \(8.40\pm7.27\) & \(89.32\pm7.40\) & \(96.67\pm5.77\) & \(598.37\pm54.26\) \\
 &  &  & IID & \(77.00\pm4.30\) & \(72.97\pm9.13\) & \(82.92\pm5.64\) & \(90.00\pm5.00\) & \(638.79\pm38.71\) & \(77.00\pm4.30\) & \(72.97\pm9.13\) & \(23.67\pm18.80\) & \(82.92\pm5.64\) & \(90.00\pm5.00\) & \(633.53\pm41.02\) \\
\addlinespace[1.5pt]
\cmidrule(lr){2-15}
 & \multirow{6}{*}{FMNIST} & \multirow{2}{*}{2} & Non-IID & \(67.96\pm3.40\) & \(74.57\pm2.41\) & \(93.12\pm2.76\) & \(98.33\pm2.89\) & \(337.72\pm5.62\) & \(67.96\pm3.40\) & \(74.57\pm2.41\) & \(54.05\pm2.70\) & \(93.12\pm2.76\) & \(98.33\pm2.89\) & \(330.45\pm4.46\) \\
 &  &  & IID & \(73.51\pm1.08\) & \(80.15\pm2.81\) & \(86.71\pm4.83\) & \(93.33\pm7.64\) & \(339.10\pm1.89\) & \(73.51\pm1.08\) & \(80.15\pm2.81\) & \(63.23\pm7.11\) & \(86.71\pm4.83\) & \(93.33\pm7.64\) & \(335.95\pm2.93\) \\
\addlinespace[0.5pt]
 &  & \multirow{2}{*}{3} & Non-IID & \(51.71\pm13.99\) & \(59.89\pm5.21\) & \(80.28\pm2.75\) & \(86.67\pm5.77\) & \(397.94\pm16.16\) & \(51.71\pm13.99\) & \(59.89\pm5.21\) & \(11.91\pm12.99\) & \(80.28\pm2.75\) & \(86.67\pm5.77\) & \(396.06\pm15.21\) \\
 &  &  & IID & \(70.18\pm1.96\) & \(77.00\pm5.75\) & \(89.93\pm7.34\) & \(95.00\pm5.00\) & \(468.34\pm13.10\) & \(70.18\pm1.96\) & \(77.00\pm5.75\) & \(49.06\pm16.04\) & \(89.93\pm7.34\) & \(95.00\pm5.00\) & \(439.55\pm12.62\) \\
\addlinespace[0.5pt]
 &  & \multirow{2}{*}{5} & Non-IID & \(52.11\pm7.67\) & \(58.31\pm6.25\) & \(83.53\pm12.06\) & \(88.33\pm12.58\) & \(618.10\pm38.92\) & \(52.11\pm7.67\) & \(58.31\pm6.25\) & \(12.33\pm7.18\) & \(83.53\pm12.06\) & \(88.33\pm12.58\) & \(584.64\pm28.46\) \\
 &  &  & IID & \(63.27\pm5.98\) & \(70.86\pm5.64\) & \(86.67\pm8.38\) & \(93.33\pm7.64\) & \(660.88\pm35.31\) & \(63.27\pm5.98\) & \(70.86\pm5.64\) & \(21.89\pm15.64\) & \(86.67\pm8.38\) & \(93.33\pm7.64\) & \(630.57\pm29.41\) \\
\addlinespace[1.5pt]
\midrule
\multirow{12}{*}{\(1\times128\)} & \multirow{6}{*}{MNIST} & \multirow{2}{*}{2} & Non-IID & \(82.82\pm3.66\) & \(76.49\pm2.36\) & \(86.55\pm2.70\) & \(95.00\pm0.00\) & \(691.72\pm3.39\) & \(82.82\pm3.66\) & \(76.49\pm2.36\) & \(57.47\pm2.48\) & \(86.55\pm2.70\) & \(95.00\pm0.00\) & \(687.16\pm4.07\) \\
 &  &  & IID & \(87.16\pm3.48\) & \(86.58\pm8.75\) & \(86.50\pm7.15\) & \(96.68\pm2.88\) & \(1383.51\pm1160.76\) & \(87.16\pm3.48\) & \(86.58\pm8.75\) & \(76.04\pm13.78\) & \(86.50\pm7.15\) & \(96.68\pm2.88\) & \(732.05\pm41.15\) \\
\addlinespace[0.5pt]
 &  & \multirow{2}{*}{3} & Non-IID & \(79.22\pm4.24\) & \(71.24\pm14.78\) & \(86.68\pm5.25\) & \(100.00\pm0.00\) & \(961.69\pm73.14\) & \(79.22\pm4.24\) & \(71.24\pm14.78\) & \(34.95\pm26.79\) & \(86.68\pm5.25\) & \(100.00\pm0.00\) & \(901.83\pm82.10\) \\
 &  &  & IID & \(86.47\pm1.12\) & \(86.28\pm3.59\) & \(92.74\pm0.00\) & \(100.00\pm0.00\) & \(937.33\pm17.54\) & \(86.47\pm1.12\) & \(86.28\pm3.59\) & \(65.75\pm8.39\) & \(92.74\pm0.00\) & \(100.00\pm0.00\) & \(937.97\pm14.50\) \\
\addlinespace[0.5pt]
 &  & \multirow{2}{*}{5} & Non-IID & \(70.24\pm8.49\) & \(53.35\pm3.78\) & \(86.54\pm2.69\) & \(98.33\pm2.89\) & \(1157.96\pm38.24\) & \(70.24\pm8.49\) & \(53.35\pm3.78\) & \(3.55\pm3.49\) & \(86.54\pm2.69\) & \(98.33\pm2.89\) & \(1172.40\pm49.15\) \\
 &  &  & IID & \(82.76\pm2.39\) & \(80.22\pm6.97\) & \(92.74\pm0.00\) & \(100.00\pm0.00\) & \(1367.73\pm75.64\) & \(82.76\pm2.39\) & \(80.22\pm6.97\) & \(41.50\pm19.48\) & \(92.74\pm0.00\) & \(100.00\pm0.00\) & \(1374.61\pm81.24\) \\
\addlinespace[1.5pt]
\cmidrule(lr){2-15}
 & \multirow{6}{*}{FMNIST} & \multirow{2}{*}{2} & Non-IID & \(70.93\pm3.35\) & \(71.37\pm9.73\) & \(82.96\pm0.15\) & \(95.00\pm0.00\) & \(669.95\pm24.46\) & \(70.93\pm3.35\) & \(71.37\pm9.73\) & \(47.63\pm15.72\) & \(82.96\pm0.15\) & \(95.00\pm0.00\) & \(666.94\pm28.03\) \\
 &  &  & IID & \(78.76\pm0.20\) & \(79.90\pm4.84\) & \(81.09\pm2.67\) & \(93.33\pm2.89\) & \(696.31\pm22.22\) & \(78.76\pm0.20\) & \(79.90\pm4.84\) & \(67.90\pm9.99\) & \(81.09\pm2.67\) & \(93.33\pm2.89\) & \(694.15\pm21.34\) \\
\addlinespace[0.5pt]
 &  & \multirow{2}{*}{3} & Non-IID & \(65.11\pm0.04\) & \(70.85\pm4.94\) & \(84.42\pm5.39\) & \(98.33\pm2.89\) & \(4660.86\pm6543.80\) & \(65.11\pm0.04\) & \(70.85\pm4.94\) & \(43.09\pm9.92\) & \(84.42\pm5.39\) & \(98.33\pm2.89\) & \(889.47\pm31.83\) \\
 &  &  & IID & \(75.02\pm0.71\) & \(81.22\pm3.22\) & \(85.94\pm2.62\) & \(93.33\pm2.89\) & \(917.67\pm10.78\) & \(75.02\pm0.71\) & \(81.22\pm3.22\) & \(61.98\pm6.13\) & \(85.94\pm2.62\) & \(93.33\pm2.89\) & \(914.02\pm10.09\) \\
\addlinespace[0.5pt]
 &  & \multirow{2}{*}{5} & Non-IID & \(57.76\pm2.78\) & \(64.14\pm4.07\) & \(87.37\pm8.01\) & \(95.00\pm8.66\) & \(1247.64\pm29.05\) & \(57.76\pm2.78\) & \(64.14\pm4.07\) & \(14.92\pm10.98\) & \(87.37\pm8.01\) & \(95.00\pm8.66\) & \(1240.41\pm28.26\) \\
 &  &  & IID & \(69.11\pm2.71\) & \(79.91\pm4.67\) & \(88.88\pm5.40\) & \(98.33\pm2.89\) & \(1371.36\pm58.89\) & \(69.11\pm2.71\) & \(79.91\pm4.67\) & \(41.46\pm15.27\) & \(88.88\pm5.40\) & \(98.33\pm2.89\) & \(1371.60\pm72.13\) \\
\addlinespace[1.5pt]
\midrule
\multirow{12}{*}{\(2\times64\)} & \multirow{6}{*}{MNIST} & \multirow{2}{*}{2} & Non-IID & \(71.98\pm4.97\) & \(77.79\pm3.70\) & \(84.81\pm12.17\) & \(93.33\pm7.64\) & \(376.46\pm10.73\) & \(73.02\pm6.11\) & \(68.08\pm8.99\) & \(40.11\pm20.74\) & \(78.58\pm2.67\) & \(90.00\pm5.00\) & \(349.64\pm22.42\) \\
 &  &  & IID & \(77.44\pm10.12\) & \(80.21\pm6.94\) & \(78.41\pm5.49\) & \(85.00\pm5.00\) & \(373.59\pm13.96\) & \(83.87\pm3.54\) & \(74.50\pm8.26\) & \(55.47\pm12.20\) & \(78.53\pm9.98\) & \(90.00\pm10.00\) & \(371.13\pm23.70\) \\
\addlinespace[0.5pt]
 &  & \multirow{2}{*}{3} & Non-IID & \(68.51\pm7.15\) & \(58.70\pm6.49\) & \(68.89\pm15.39\) & \(73.33\pm15.28\) & \(433.90\pm31.09\) & \(66.42\pm6.41\) & \(54.95\pm3.69\) & \(5.95\pm6.77\) & \(83.34\pm7.33\) & \(95.00\pm0.00\) & \(428.98\pm7.11\) \\
 &  &  & IID & \(79.98\pm4.01\) & \(69.97\pm1.84\) & \(78.38\pm2.71\) & \(83.33\pm2.89\) & \(453.33\pm7.22\) & \(81.93\pm4.80\) & \(61.94\pm2.42\) & \(26.70\pm13.29\) & \(72.11\pm12.69\) & \(80.00\pm13.23\) & \(441.32\pm8.69\) \\
\addlinespace[0.5pt]
 &  & \multirow{2}{*}{5} & Non-IID & \(52.13\pm14.45\) & \(42.63\pm5.89\) & \(60.82\pm13.79\) & \(71.67\pm11.55\) & \(547.16\pm34.11\) & \(57.93\pm7.56\) & \(38.82\pm7.31\) & \(1.38\pm2.39\) & \(59.35\pm24.16\) & \(71.67\pm15.28\) & \(538.06\pm42.18\) \\
 &  &  & IID & \(69.84\pm4.66\) & \(62.74\pm10.26\) & \(79.97\pm2.76\) & \(86.67\pm7.64\) & \(648.94\pm46.73\) & \(73.98\pm5.85\) & \(48.04\pm7.74\) & \(5.47\pm6.30\) & \(65.72\pm7.26\) & \(83.33\pm2.89\) & \(585.62\pm30.34\) \\
\addlinespace[1.5pt]
\cmidrule(lr){2-15}
 & \multirow{6}{*}{FMNIST} & \multirow{2}{*}{2} & Non-IID & \(54.62\pm8.05\) & \(62.56\pm4.78\) & \(72.05\pm12.66\) & \(85.00\pm13.23\) & \(342.48\pm16.05\) & \(54.62\pm8.05\) & \(62.56\pm4.78\) & \(38.61\pm13.36\) & \(72.05\pm12.66\) & \(85.00\pm13.23\) & \(336.96\pm16.17\) \\
 &  &  & IID & \(72.60\pm1.79\) & \(72.17\pm6.85\) & \(74.40\pm10.34\) & \(87.50\pm3.54\) & \(368.19\pm19.31\) & \(73.13\pm1.57\) & \(72.96\pm5.03\) & \(57.06\pm14.23\) & \(75.20\pm7.44\) & \(90.00\pm5.00\) & \(770.04\pm715.42\) \\
\addlinespace[0.5pt]
 &  & \multirow{2}{*}{3} & Non-IID & \(48.78\pm9.71\) & \(63.01\pm2.44\) & \(72.10\pm4.79\) & \(83.33\pm2.89\) & \(453.48\pm9.66\) & \(48.78\pm9.71\) & \(63.01\pm2.44\) & \(29.66\pm6.82\) & \(72.10\pm4.79\) & \(83.33\pm2.89\) & \(448.57\pm3.78\) \\
 &  &  & IID & \(67.13\pm0.66\) & \(67.28\pm2.24\) & \(71.94\pm13.68\) & \(82.50\pm3.54\) & \(478.99\pm18.62\) & \(68.13\pm1.79\) & \(68.34\pm2.43\) & \(47.31\pm6.71\) & \(71.96\pm9.67\) & \(83.33\pm2.89\) & \(452.17\pm1.70\) \\
\addlinespace[0.5pt]
 &  & \multirow{2}{*}{5} & Non-IID & \(45.93\pm5.39\) & \(47.04\pm7.82\) & \(65.66\pm23.64\) & \(81.67\pm18.93\) & \(574.41\pm36.30\) & \(45.93\pm5.39\) & \(47.04\pm7.82\) & \(2.80\pm4.84\) & \(65.66\pm23.64\) & \(81.67\pm18.93\) & \(573.29\pm36.49\) \\
 &  &  & IID & \(60.57\pm4.01\) & \(58.14\pm7.53\) & \(79.33\pm3.33\) & \(92.50\pm10.61\) & \(667.47\pm59.22\) & \(61.78\pm3.53\) & \(62.46\pm9.19\) & \(20.42\pm21.79\) & \(78.52\pm2.74\) & \(91.67\pm7.64\) & \(662.59\pm31.35\) \\
\bottomrule
\end{tabular}%
}
\end{table*}

The three reported certification quantities have different probabilistic
interpretations. Local certification evaluates individual client posteriors,
whereas transported certification evaluates whether joint local events remain
safe after aggregation. Direct certification instead searches safe regions
directly under the induced global posterior. Figure~\ref{fig:bound_hierarchy}
illustrates this distinction.
\begin{figure}[t]
    \centering
    \includegraphics[height=5.5cm, width=\columnwidth]
   {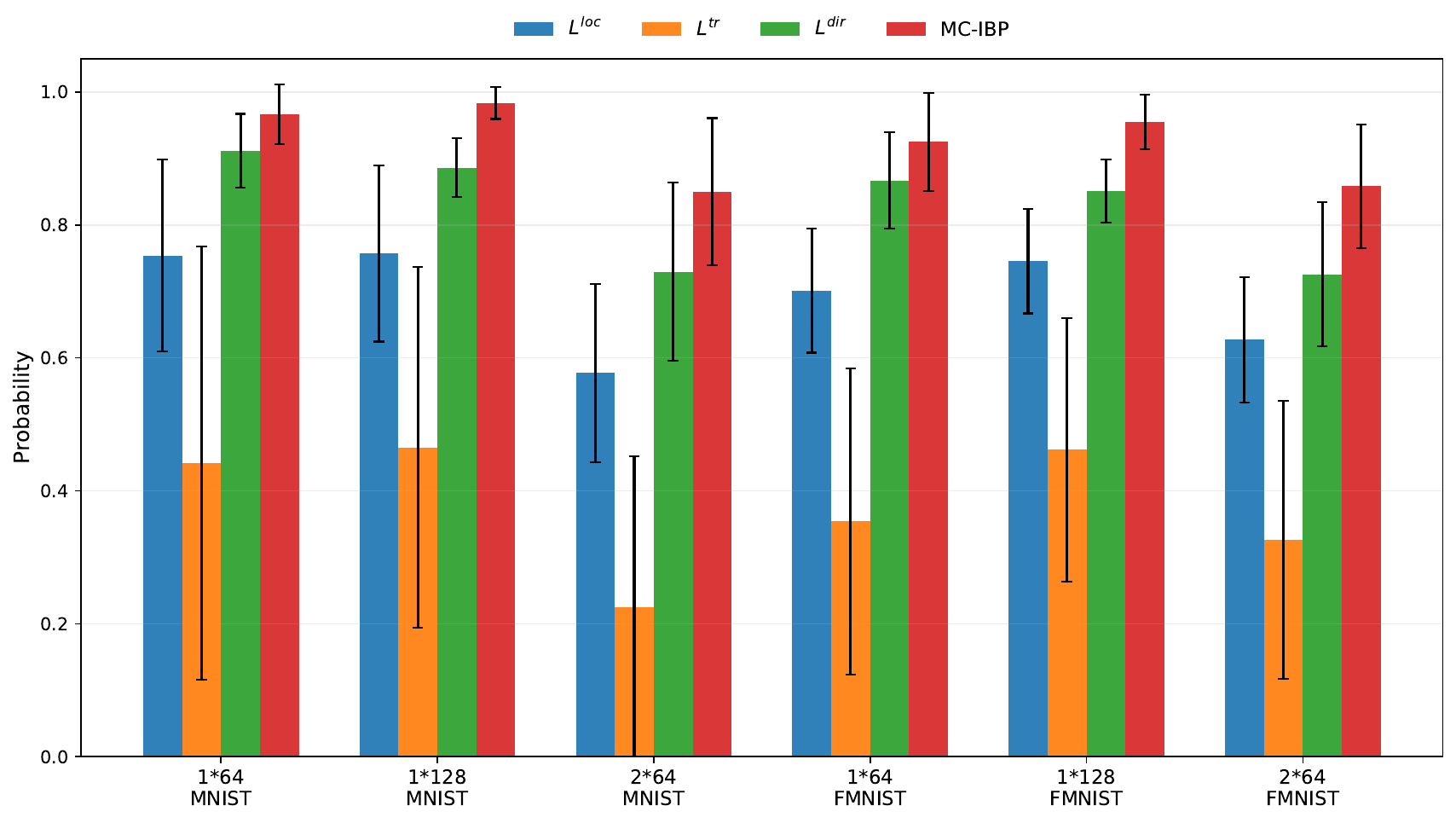}
\caption{FedAvg bound hierarchy over dataset--architecture pairs, comparing $L^{loc}$, $L^{tr}$, $L^{dir}$, and MC-IBP acceptance.}
    \label{fig:bound_hierarchy}
\end{figure}
Figure~\ref{fig:bound_hierarchy} summarizes the relative behavior of the local,
transported, and direct certification bounds across different dataset and
architecture configurations.

\noindent\textbf{a) Direct certification:}
Direct certificates are non-vacuous for every dataset, architecture, and
global law in Table~\ref{tab:heterogeneity_grouped}.
Figure~\ref{fig:calibration} shows a clear positive association between direct
certification and MC--IBP, with all evaluated points above the equality line.
The certified posterior mass ranges
from \(72.05\%\) to \(91.39\%\) for PoG and from \(72.58\%\) to
\(91.20\%\) for FedAvg. MC--IBP is larger than the corresponding direct
certificate in the reported configurations, with differences between
\(5.47\) and \(13.25\) percentage points. This ordering is consistent with
finite safe-cell coverage and the stronger requirement imposed by
region-based certification: a box contributes only when every weight vector
inside it is verified. The difference should not be interpreted as an error
bound on the true safety probability, because MC--IBP is itself based on an
incomplete verifier and a finite posterior sample.
The smallest direct-to-MC differences occur for the \(1\times64\)
architecture. The largest differences occur for the \(2\times64\) model on
Fashion-MNIST. The effect is not completely monotone across architectures:
for example, PoG on MNIST exhibits a smaller difference for \(2\times64\)
than for \(1\times128\). The results therefore indicate architecture-dependent
certifiability, but do not support a claim that depth alone determines
certificate tightness.

\begin{figure}[t]
    \centering
     \includegraphics[height=5.5cm, width=\columnwidth]
   {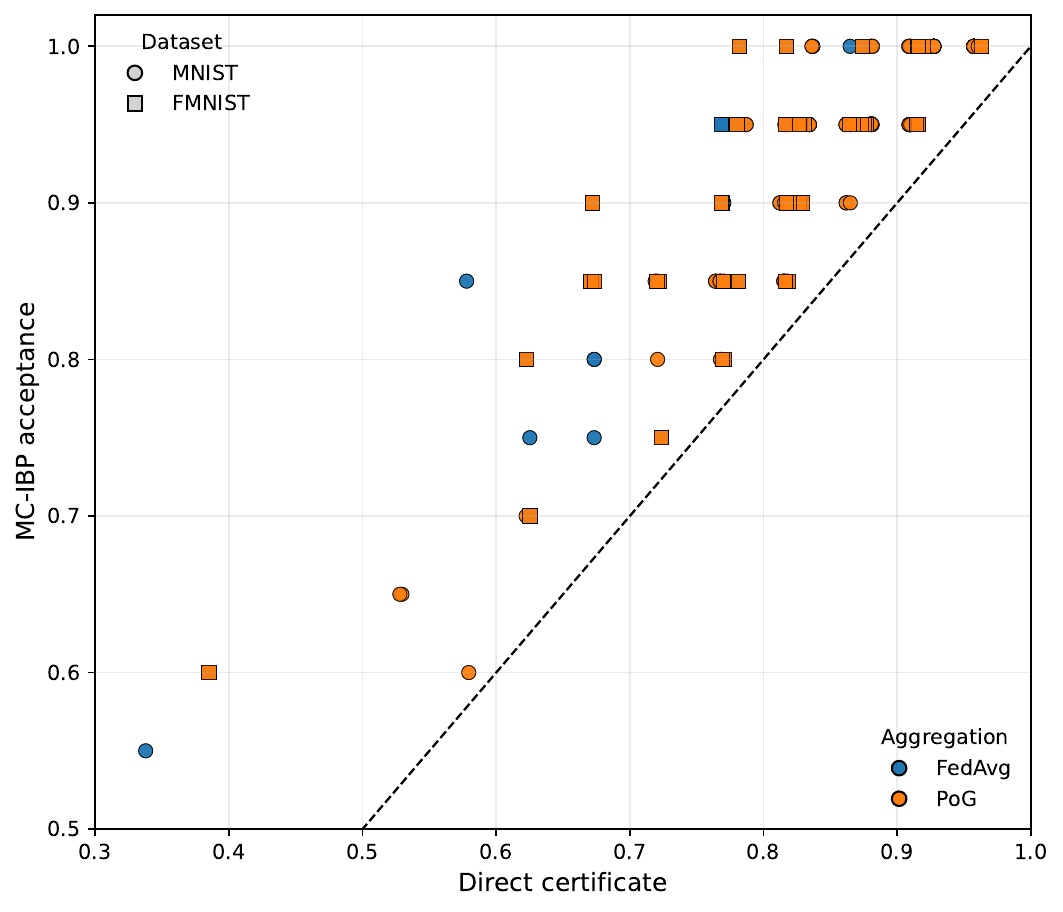}
\caption{Direct certified mass versus MC--IBP acceptance across all evaluated
configurations. The dashed line denotes equality.}
    \label{fig:calibration}
\end{figure}

\noindent\textbf{b) Transported FedAvg certificate:}
The transported certificate is positive in all six dataset--architecture FedAvg configurations, ranging from
\(22.51\%\) to \(46.89\%\). It is more conservative than the direct
FedAvg certificate because certified mass must be represented by retained local
cells, included in the searched tuple set, and verified after transporting the
complete Cartesian product through the FedAvg map.
The transported certificate also varies substantially more across federation
configurations. Its standard deviation ranges from \(18.09\) to \(33.06\)
percentage points, whereas the standard deviation of the direct FedAvg
certificate ranges from \(4.02\) to \(13.59\) points. This larger variability
indicates that transported coverage is more sensitive than direct global
certification to the federation configuration, the retained local cells, and
whether their complete aggregation images remain verifiable. Finite tuple
evaluation may introduce further variability.
Using the aggregate means, the transported certificate ratio relative to the direct FedAvg
certificate ranges approximately from \(31\%\) to \(57\%\). The
smallest ratio occurs for MNIST \(2\times64\), while the largest occurs for
Fashion-MNIST \(1\times128\). These ratios are descriptive measures of the
selected certificate constructions and should not be interpreted as the
fraction of the true safe probability recovered by the method.
To further examine this sensitivity, Fig.~\ref{fig:transport_retention}
reports the transported-to-direct certificate ratio,
\(100\times L^{\mathrm{tr}}/L^{\mathrm{dir}}\), as a function of the
Dirichlet concentration parameter and the number of participating clients.
The retention ratio generally improves as \(\alpha\) increases, indicating
that less heterogeneous client data lead to better agreement between the
transported and direct FedAvg certificates. In contrast, increasing the number
of clients consistently reduces the recovered fraction of direct certified
mass. This pattern suggests that the gap between transported and direct
certification is shaped not only by verifier conservatism, but also by the
geometric compatibility of certified local regions after aggregation. As the
number of clients grows, the Cartesian product of retained local events
becomes more difficult to verify after transport, leading to a larger loss in
certified mass.

\begin{figure}[t]
    \centering
    \includegraphics[width=\columnwidth]{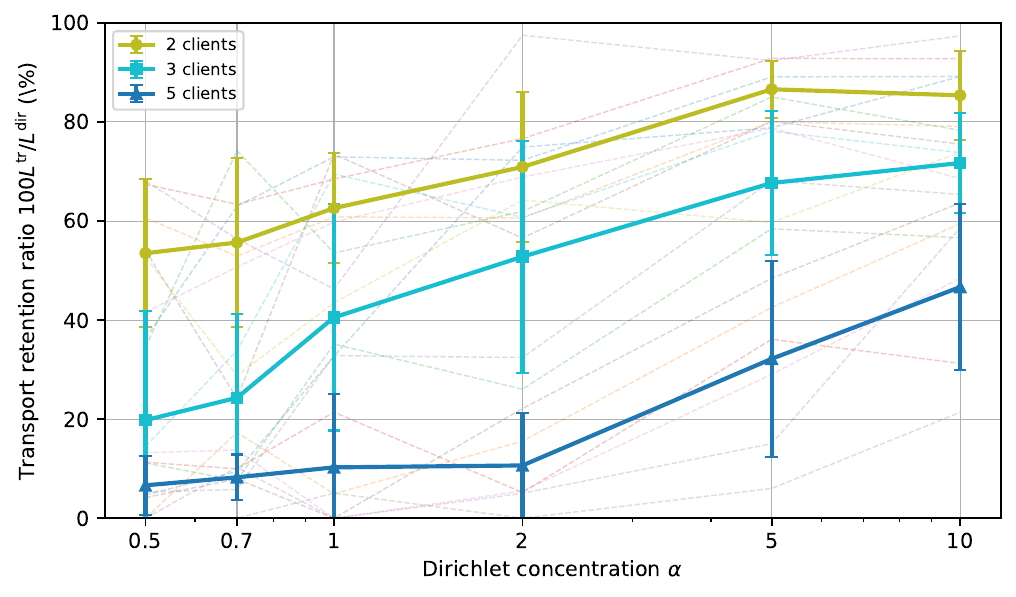}
    \caption{Transport retention ratio (\(100L^{\mathrm{tr}}/L^{\mathrm{dir}}\)) versus Dirichlet concentration for different federation scales. Thin dashed curves show individual dataset--architecture configurations, while the bold curves with error bars show the mean and standard deviation across all six configurations.}
    \label{fig:transport_retention}
\end{figure}

\noindent\textbf{c) Architecture and certification:}
The \(1\times128\) architecture yields the highest mean predictive accuracy
under both global laws and on both datasets. It does not, however, yield the
largest certified posterior mass. The smaller \(1\times64\) network achieves
the strongest direct certificate in all four dataset--method comparisons.
Increasing width from \(1\times64\) to \(1\times128\) improves accuracy but
typically reduces direct certification by several percentage points.
The \(2\times64\) architecture produces a larger reduction in certified mass
and, for FedAvg, the smallest transported certificate on MNIST. This pattern is consistent with the behavior of interval verification
\cite{gowal2018effectiveness,mirman2018differentiable}: in a larger parameter
space, axis-aligned cells may capture less useful posterior mass while remaining
certifiable, and additional layers compound dependency loss between uncertain
weights and hidden activations. The experiment does not
fully separate width and depth effects, but it shows that predictive accuracy
and certifiable posterior mass need not improve together.

\noindent\textbf{d) Dataset effects:}
Fashion-MNIST reduces mean accuracy by approximately \(11\)--\(14\)
percentage points relative to MNIST. The corresponding reduction in direct
certification is smaller. For the one-layer architectures, the decrease is
typically \(3\)--\(5\) points; for \(2\times64\), the direct FedAvg
certificate is nearly unchanged across the two datasets.
This contrast does not imply that the two datasets have comparable global
safety. Properties are defined only for selected correctly classified inputs,
and results are macro-averaged across those properties. The observation instead
shows that test accuracy and the posterior mass satisfying the selected local
robustness specifications measure different aspects of model behavior.

\noindent\textbf{e) PoG and FedAvg:}
Neither global law dominates across all architectures and metrics. For
\(1\times64\), their direct certificates are nearly identical: the difference
is below \(0.25\) points on both datasets. For \(1\times128\), PoG provides
higher accuracy, direct certification, and MC--IBP on MNIST and Fashion-MNIST.
Its direct-certificate advantage is \(3.86\) points on MNIST and \(3.51\)
points on Fashion-MNIST.
The ordering changes for \(2\times64\). FedAvg obtains higher accuracy and
MC--IBP on both datasets. PoG retains a \(2.27\)-point direct-certificate
advantage on MNIST, but is \(0.53\) points lower on Fashion-MNIST. These
architecture-dependent differences do not support a uniform ranking between
the two fusion rules.
Because the posterior standard deviation is fixed after local training, the
results should not be interpreted as evidence that one method provides better
Bayesian uncertainty calibration. They only compare predictive behavior and
certification under the global Gaussian laws used in this controlled study.

\noindent\textbf{f) Federation configuration and computational cost:}
The pooled statistics do not establish a monotone relationship between
certification and either client count or the Dirichlet concentration parameter.
Increasing the number of clients can contract the FedAvg pushforward covariance,
which may improve direct certification, while simultaneously decreasing tuple
masses and enlarging the Cartesian tuple space. These effects act in opposite
directions and cannot be separated from the aggregated table alone.
The \(1\times128\) architecture is the most computationally expensive.
The PoG runtime for Fashion-MNIST \(1\times128\) has a particularly large
standard deviation, \(1593.97\pm2664.92\) seconds, indicating a strongly
skewed runtime distribution or expensive outlying configurations. Mean runtime
alone is therefore insufficient for claiming a consistent computational
advantage between PoG and FedAvg.

\section{Conclusion}
\label{sec:conclusion}

We formulated probabilistic safety certification for one-shot federated Bayesian neural networks under the probability law induced by the deployment rule. The proposed construction transports disjoint local posterior events through the aggregation map and includes their probability only when the complete image is verified safe. For FedAvg with nonnegative weights, the image of axis-aligned cells is exact, so the aggregation step itself introduces no additional geometric relaxation. The experiments produce non-vacuous transported certificates across all
reported dataset--architecture FedAvg aggregates and show that predictive accuracy, direct certified mass, and transported coverage need not follow the same ordering, while neither FedAvg nor PoG uniformly dominates across architectures. Although the theoretical construction is not restricted to the evaluated networks, extending the current implementation to larger models will require tighter weight-space verification and scalable handling of local-event combinations.

\section{Future Work}
\label{sec:Future Work}

While the proposed framework provides a mathematically sound and exact geometric certificate for one-shot FedAvg deployment, evaluating the full Cartesian product of local posterior events introduces a combinatorial challenge. The joint tuple space grows exponentially with the number of participating clients, which limits exhaustive evaluation to smaller federated settings. Future work will therefore focus on replacing exhaustive tuple evaluation with more scalable search strategies without sacrificing the theoretical guarantees. Promising directions include distributed primal-dual optimization, multi-agent reinforcement learning, and stochastic generalized Nash equilibrium methods to identify and prune unverifiable parameter combinations before server-side aggregation.

\bibliographystyle{IEEEtran}
\bibliography{TAIbib}

\end{document}